\documentclass{article}

\PassOptionsToPackage{numbers, compress}{natbib}
\usepackage[preprint]{neurips_2026}

\usepackage{array}
\usepackage[utf8]{inputenc}
\usepackage[T1]{fontenc}
\usepackage[
    colorlinks=true,
    linkcolor=blue,
    citecolor=blue,
    urlcolor=blue
]{hyperref}
\usepackage{url}
\usepackage{booktabs}
\usepackage{amsfonts}
\usepackage{nicefrac}
\usepackage{microtype}
\usepackage{xcolor}
\usepackage{graphicx}
\usepackage{amssymb}
\usepackage{amsmath}
\usepackage{multirow}
\usepackage{makecell}
\usepackage{enumitem}
\usepackage{tabularx}
\usepackage{subcaption}
\usepackage{caption}
\usepackage[capitalize,noabbrev]{cleveref}

\usepackage{amssymb}
\usepackage{mathtools}
\usepackage{amsthm}

\usepackage[capitalize,noabbrev]{cleveref}
\usepackage{amssymb}
\usepackage{amsmath}
\usepackage{bbm}

\theoremstyle{plain}
\newtheorem{theorem}{Theorem}[section]
\newtheorem{proposition}[theorem]{Proposition}

\theoremstyle{definition}

\theoremstyle{remark}

\title{Off-Policy Learning to Reason Works Because It Is More Pessimistic Than You Think}

\author{%
  Otmane Sakhi\\
  Criteo AI Lab \\
  \texttt{o.sakhi@criteo.com} \\
  \And
  Aleksei Arzhantsev\\
  Criteo AI Lab, CREST IP Paris \\
  \texttt{a.arzhantsev@criteo.com} \\
  \And
  Imad Aouali \\
  Criteo AI Lab \\
  \texttt{i.aouali@criteo.com}
  \And
  Flavian Vasile \\
  Criteo AI Lab \\
  \texttt{f.vasile@criteo.com}
}

\begin{document}
\maketitle

\begin{abstract}

Large scale reinforcement learning has become a central tool for improving reasoning in large language models. At this scale, generation is often lagged or asynchronous, so updates are performed on data collected by older policies. This makes learning inherently off-policy. Most existing approaches nevertheless remain rooted in PPO-style trust-region objectives, treating training as approximately on-policy and using importance weights to correct distribution mismatch. These corrections can introduce high variance, destabilize optimization, and accelerate entropy collapse. Recent work suggests an alternative: rather than correcting the mismatch, one can embrace off-policy data and remove importance weights, often yielding stronger algorithms. In this paper, we provide an intuitive construction of off-policy objectives that include successful off-policy objectives and show that their effectiveness can be understood through implicit pessimism: they optimize toward target policies that are more conservative than their nominal objectives suggest. This perspective explains why some particular implementation choices improve stability: they implicitly control the effective target distribution. We then propose a principled modification that stabilize this induced distribution and improve off-policy learning.
\end{abstract}


\section{Introduction}
\label{sec:intro}

Reinforcement learning (RL) has emerged as a fundamental tool for improving
the reasoning capabilities of large language models (LLMs)~\cite{openai2026openaio1card, deepseekai2025, yu2026dapo}.
By optimizing models against task-specific or verifiable rewards, RL enables
gains beyond supervised fine-tuning~\cite{openai2026openaio1card, deepseekai2025}, particularly on complex
reasoning tasks such as mathematics, coding, and multi-step
problem solving~\cite{NEURIPS2022_18abbeef, deepseekai2025, kimi}.

The reinforcement learning pipelines used for modern language models,
however, differ significantly from the classical online setting. At large
scale, generating rollouts for reasoning models is expensive. Completions are often produced
asynchronously, by lagged inference workers, or by checkpoints that are
older than the policy currently being updated. As a result, the data used for
training are not sampled from the policy being optimized, but from an
older behavior policy. Large-scale reinforcement learning for language
model reasoning is therefore inherently off-policy \cite{kimi, noukhovitch2025faster, llmscanlearntoreasonfromoffpolicydata}.

Despite this, many existing methods are still motivated by on-policy or
near-on-policy trust-region algorithms \cite{pmlr-v37-schulman15}, such as PPO \cite{ppo} and GRPO \cite{deepseek-math}. These
methods attempt to correct the mismatch between the behavior policy and
the learned policy using importance weighting, clipping , filtering, or
related heuristics \cite{ppo, deepseek-math, yu2026dapo}. While these corrections are effective in some
settings, they can introduce high-variance updates and become unstable
when the policy lag is large. In practice, the mechanisms introduced to
control these ratios often become central to training stability, rather
than secondary implementation details.

Recent work suggests a different approach: instead of correcting
off-policy data to make it look on-policy, one can train directly on data
sampled from the behavior policy and remove importance weights entirely.
$A^\star$-PO \cite{brantley2026accelerating} and more recently OAPL \cite{llmscanlearntoreasonfromoffpolicydata} are a representative example of this direction. Both optimize a simple
regression objective, and has been
shown to work well in large-scale reasoning settings. However, the reason
for its success is not fully understood. In particular, their nominal derivations suggest one target distribution, while practical
implementations decouple the
temperature used for advantage computation from the temperature used for
regularization. These choices are empirically important, but their role is
not explained by the nominal derivation.

In this paper, we study ratio-free off-policy objectives.
Our goal is not to derive another trust-region surrogate, but to
understand what target policy these objectives actually optimize toward
when trained on data from an older policy. We show that the answer is
different from the usual exponential target associated with
KL-regularized policy optimization \cite{rafailov2023direct, pmlr-v37-schulman15}. Instead, these objectives induce a
Lambert-tempered target policy \cite{huang2025correcting}. This target can be more conservative than
the nominal exponential objective suggests, especially on high-advantage
samples and small regularizations. 

This observation provides a simple explanation for why off-policy
learning without importance weights can work well. The objective is not
merely ignoring the distribution mismatch. Rather, under the right
advantage normalization, the mismatch induces a form of implicit
pessimism \cite{jin2021pessimism, cql, NEURIPS2024_9379ea6b}: the learned policy increases the probability of good samples,
but does not over-commit exponentially to the highest observed rewards.
This is particularly valuable in language-model reasoning, where rollouts are stale and a small number of sampled completions
can otherwise dominate the update.

Our analysis also reveals that this pessimism is not automatic. The
target policy depends sensitively on how the advantage is normalized. We
identify three qualitative regimes. In the positive regime, the target is
conservative and tempers large advantages. At the boundary, the target
recovers the usual exponential update, which is fragile because small
changes in the advantage baseline can move the objective into a different
regime. In the negative regime, the target becomes more aggressive than
the exponential update, making it more prone to entropy collapse.

This perspective clarifies the behavior of OAPL. We show that the
advantage estimator suggested by its nominal theory can induce an
unstable target at the population level when the group size is finite,
even though it is normalized on each sampled group. This explains why the implementation of $A^\star$-PO \cite{brantley2026accelerating} and OAPL \cite{llmscanlearntoreasonfromoffpolicydata} deviates from the theory by using a
larger temperature $\beta$ for advantage computation than for regularization. Our
analysis shows that this choice moves the advantage closer to a centered
group advantage, which pushes the induced target back toward the pessimistic regime.

Motivated by this understanding, we propose a simpler, more principled
modification. Instead of relying on temperature decoupling and hyperparameter tuning as an
implementation heuristic, we directly choose the advantage baseline so
that the induced target lies in the pessimistic Lambert regime. The
resulting rule is a $\beta$-shifted mean advantage: start from the usual
group-centered reward advantage and shift it upward by $\beta$. This choice is easy to implement, introduces no additional
temperature hyperparameter, and guarantees that the induced Lambert target
is in the desired conservative regime.

Empirically, this modification improves the robustness of off-policy
learning. Compared with OAPL, the resulting method is less sensitive to
the regularization strength and remains more stable under larger policy
lags. In our experiments, it maintains entropy more reliably and avoids
the collapse or instability observed in OAPL under aggressive
regularization or stale rollouts.

\textbf{Contributions.} Our contributions are as follows:
\begin{itemize}
    \item We analyze a class of ratio-free off-policy objectives for language-model reasoning, including the objective underlying $A^\star$-PO/OAPL, and show that these objectives induce Lambert-tempered target policies whose behavior is controlled by an implicit pessimism parameter determined by the advantage normalization.

    \item We use this characterization to explain why the nominal $A^\star$-PO/OAPL
    advantage can induce an unstable target, and why the practical
    temperature decoupling used improves stability. We then propose a simple $\beta$-shifted-mean advantage that targets the
    pessimistic Lambert regime by construction, removing the need for the
    temperature-decoupling heuristic.

    \item We empirically show that this modification improves stability
    and robustness under varying temperatures and policy
    lags.
\end{itemize}

\paragraph{Outline.} The rest of the paper is organized as follows. Section~\ref{sec:preliminaries} introduces the
reinforcement learning setting for language-model reasoning and reviews
standard trust-region methods. Section~\ref{sec:mopo} studies ratio-free off-policy
objectives and derives their induced target policies. Section~\ref{sec:pess} analyzes
the resulting Lambert regimes, explains the behavior of OAPL, and
introduces our shifted-mean advantage. Section~\ref{sec:experiments} presents empirical
comparisons under different regularization strengths and policy lags while Section~\ref{sec:conclusion} concludes.

\section{Preliminaries}
\label{sec:preliminaries}

We consider reinforcement learning for language-model reasoning. Let
$x\sim\mathcal{D}$ denote a prompt and let
$y=(y_1,\ldots,y_T)$ denote a sampled completion. A policy
$\pi_\theta$ defines the autoregressive distribution
\[
    \pi_\theta(y\mid x)
    =
    \prod_{t=1}^{T}
    \pi_\theta(y_t\mid x,y_{<t}).
\]
After generating a completion, the model receives a scalar, bounded reward
$r(x,y) \in [0,1]$, for example from a verifier, a unit test, or a reward model.
For a group of $G \ge 2$ completions $y_{1:G}$ sampled for the same prompt, we
write $r_i=r(x,y_i)$ and use $\widehat A_i$ to denote a group-relative
score indicating whether completion $y_i$ should be upweighted or
downweighted relative to the group.

\paragraph{Trust-region optimization and GRPO.} Trust-region optimization has established itself as a staple for large scale, deep RL~\cite{deepseek-math, kimi, yang2025qwen3}. The idea is to optimize a policy that improves $\pi_{\mathrm{old}}$ while staying in its vicinity, in the KL regularized sense.
Group Relative Policy Optimization (GRPO) can be viewed as a PPO-style
trust-region method adapted to the LLM setting~\citep{ppo, deepseek-math}.
Given a behavior policy $\pi_{\mathrm{old}}$ used to generate completions,
GRPO optimizes a clipped surrogate objective of the form
\begin{align}
    \mathbb{E}_{y_{1:G}\sim \pi_{\mathrm{old}}(\cdot\mid x)}
    \left[
    \frac{1}{G}
    \sum_{i=1}^{G}
    \frac{1}{T_i}
    \sum_{t=1}^{T_i}
    \min\left(
        \rho_{i,t}(\theta)\widehat A_i,\,
        \mathrm{clip}\!\left(
            \rho_{i,t}(\theta),1-\epsilon,1+\epsilon
        \right)\widehat A_i
    \right)
    \right],
    \label{eq:grpo}
\end{align}
where
\[
    \rho_{i,t}(\theta)
    =
    \frac{
        \pi_\theta(y_{i,t}\mid x,y_{i,<t})
    }{
        \pi_{\mathrm{old}}(y_{i,t}\mid x,y_{i,<t})
    },
    \qquad
    \widehat A_i
    =
    \frac{r_i-\bar r_G}{\sigma_G}.
\]
The ratio $\rho_{i,t}(\theta)$ measures the change in token probability
between the policy being optimized and the behavior policy that generated
the data. The clipping operation restricts this ratio, preventing the new
policy from moving too far from $\pi_{\mathrm{old}}$ on sampled tokens.

This objective is effective when the data are fresh and the optimized
policy remains close to the behavior policy. However, large-scale LLM RL
often violates this near-on-policy assumption \cite{yu2026dapo, noukhovitch2025faster}. Indeed, rollouts are often generated by stale
checkpoints. As optimization proceeds, the current policy $\pi_\theta$ can
move far from the behavior policy $\pi_{\mathrm{old}}$, making the data
increasingly off-policy. In this regime, importance weighting becomes
high-variance, and clipping or filtering becomes a central stabilizing
heuristic rather than a minor implementation detail. 

\paragraph{Embracing Off-Policy Generations.}

A natural way to avoid importance weighting is to define the desired
policy directly as an improvement over the behavior policy
\(\pi_{\mathrm{old}}\). The dominant formulation is the KL-regularized
trust-region objective. For a fixed context \(x\), it solves
\[
    \max_{\pi}
    \;
    \mathbb E_{y\sim \pi(\cdot\mid x)}
    \left[
        r(y,x)
    \right]
    -
    \beta
    \operatorname{KL}
    \left(
        \pi(\cdot\mid x)
        \,\|\, 
        \pi_{\mathrm{old}}(\cdot\mid x)
    \right).
\]
The solution is the exponentially tilted behavior policy
\[
    \pi_\beta^\star(y\mid x)
    =
    \frac{
        \exp(r(y,x)/\beta)
    }{
        \mathbb E_{y'\sim \pi_{\mathrm{old}}(\cdot\mid x)}
        \left[
            \exp(r(y',x)/\beta)
        \right]
    }
    \pi_{\mathrm{old}}(y\mid x).
\]
Equivalently, the optimal sentence-level log-ratio satisfies
\[
    \beta
    \log
    \frac{
        \pi_\beta^\star(y\mid x)
    }{
        \pi_{\mathrm{old}}(y\mid x)
    }
    =
    r(y,x)
    -
    \beta
    \log
    \mathbb E_{y'\sim \pi_{\mathrm{old}}(\cdot\mid x)}
    \left[
        \exp(r(y',x)/\beta)
    \right].
\]
Thus the KL-regularized target can be viewed as specifying an optimal
advantage, or target log-ratio, relative to \(\pi_{\mathrm{old}}\).

In practice, the expectation above is unavailable and is estimated from a
group of completions
\(y_{1:G}\sim\pi_{\mathrm{old}}(\cdot\mid x)\), with
\(r_i=r(y_i,x)\). This gives the empirical advantage
\[
    \widehat A_i^\beta
    =
    r_i
    -
    \beta
    \log
    \left(
        \frac{1}{G}
        \sum_{j=1}^G
        \exp(r_j/\beta)
    \right).
\]
This estimator is biased for finite \(G\), due to the logarithm of the
empirical average, but it is consistent.

Recent off-policy objectives use this target without trying to correct
the behavior-policy samples into on-policy samples. Instead, they regress
the current policy log-ratio toward the KL-regularized target:
\[
    \min_{\pi_\theta}
    \;
    \mathbb E_{y_{1:G}\sim \pi_{\mathrm{old}}(\cdot\mid x)}
    \left[
        \frac{1}{G}
        \sum_{i=1}^G
        \left(
            \beta
            \log
            \frac{
                \pi_\theta(y_i\mid x)
            }{
                \pi_{\mathrm{old}}(y_i\mid x)
            }
            -
            \widehat A_i^\beta
        \right)^2
    \right].
\]
This objective embraces off-policy data rather than correcting it with
importance weights: the old policy provides the samples, and the loss
directly constrains how the learned policy should move relative to
\(\pi_{\mathrm{old}}\) on those samples.

\section{Regularized Off-Policy Optimization}
\label{sec:mopo}

\paragraph{Importance-weight-free learning can lead to behavior-tilted targets.} Instead of correcting samples from $\pi_{\mathrm{old}}$ to mimic on-policy data from
$\pi_\theta$, we first consider the most direct ratio-free objective:
weighted maximum likelihood on behavior-policy samples.

Let $G\geq 1$ and let $y_{1:G}\sim \pi_{\mathrm{old}}(\cdot\mid x)$
denote a group of completions sampled from the behavior policy. For each
completion $y_i$, let $r_i=r(x,y_i)$ be its reward and let
$u_i=u_G(r_i;r_{1:G})\geq 0$ be a nonnegative reward-dependent weight.
A natural off-policy objective is
\begin{align}
    \mathbb{E}_{y_{1:G}\sim \pi_{\mathrm{old}}(\cdot\mid x)}
    \left[
        \frac{1}{G}
        \sum_{i=1}^{G}
        u_i
        \log \pi_\theta(y_i\mid x)
    \right].
    \label{eq:weighted-mle}
\end{align}
This objective is attractive because it contains no importance-weighting
ratio. The behavior policy simply provides samples, and high-reward
samples receive larger weights.

A common choice in offline RL \cite{awc, roirl} is an exponential reward weight,
$u_i
    =
    \exp\left(
        (r_i-\bar r_G)/\eta
    \right),$
which up-weights completions that outperform the group average. This objective is desirable as it makes the behavior policy
the base measure of the target distribution. To see this, fix a prompt
$x$ and consider the population version of Eq.~\eqref{eq:weighted-mle}
over all sequence distributions $\pi(\cdot\mid x)$:
\[
    \max_{\pi\in\Delta}
    \;
    \mathbb{E}_{y\sim \pi_{\mathrm{old}}(\cdot\mid x)}
    \left[
        u(x,y)\log \pi(y\mid x)
    \right].
\]
The optimizer is $\pi^\star(y\mid x)
    \propto
    \pi_{\mathrm{old}}(y\mid x)
    u(x,y)$. Therefore, when $u(x,y)=\exp(r(x,y)/\eta)$,
\[
    \pi^\star(y\mid x)
    \propto
    \pi_{\mathrm{old}}(y\mid x)
    \exp\left(
        \frac{r(x,y)}{\eta}
    \right).
\]
Thus, weighted maximum likelihood recovers a
\emph{behavior-tilted} distribution \cite{rafailov2023direct}. This is the same qualitative form as
the exponentially tilted distributions arising from KL-regularized
trust-region objectives: the behavior policy remains the base measure, and
the reward only reweights it.


While the theoretical target is desirable, a major problem of this objective is that it can be overly aggressive in finite
samples. The expectation under $\pi_{\mathrm{old}}$ is approximated by a
small set of sampled completions, so weighted maximum likelihood can
rapidly concentrate probability mass on the observed high-weight samples.
With sharp exponential weights, this effect is amplified: the empirical
target becomes highly peaked, and optimization can collapse toward the
weighted empirical distribution. This leads to entropy collapse in practice and requires additional care to control~\citep{roirl}.

Ideally, we want an update that preserves the simplicity of
Eq.~\eqref{eq:weighted-mle} and does not require importance sampling,
while avoiding the aggressiveness of exponential behavior tilting.
We can achieve this by adding to weighted likelihood a simple quadratic
penalty in log-probability space.


\paragraph{Regularize toward \(\pi_{\mathrm{old}}\).}
Based on the previous observation, our off-policy objective implements the simple intuition of maximizing reward while staying close to \(\pi_{\mathrm{old}}\). We do so by adding a sentence-level regularizer to the weighted likelihood objective. Let \(\beta>0\). We focus on:
\begin{align}
\mathbb{E}_{y_{1:G}\sim \pi_{\mathrm{old}}(\cdot\mid x)}
    \left[
        \frac{1}{G}
        \sum_{i=1}^{G}
        \left(
            \widehat A_i
            \log \pi_\theta(y_{i}| x)
            -
            \frac{\beta}{2}
            \left(
                \log
                \frac{
                    \pi_\theta(y_{i}| x)
                }{
                    \pi_{\mathrm{old}}(y_{i}| x)
                }
            \right)^2
        \right)
    \right]\,,
    \label{eq:mopo-objective}
\end{align}

where \(\widehat A_i\) is a group-wise advantage that will be specified later. The first term is an advantage-weighted log-likelihood term. The second term
penalizes squared movement in sentence log-probability relative to the
behavior policy. Importantly, this penalty is not an importance weight:
it does not multiply the gradient by a likelihood ratio, and it does not
attempt to transform off-policy samples into on-policy samples. Instead,
it directly regularizes how far the learned policy can move from the
behavior policy on observed sentences.

This objective has a clean stationary target policy, characterized by the next proposition.

\begin{proposition}[Lambert-tempered policy -- sentence level]
\label{prop:mopo-lambert-sentence}
Fix a context \(x\), and let \(\mathcal Y(x)\) denote the set of possible
responses. Let \(A(x,y)\) be a sentence-level advantage for response
\(y\in\mathcal Y(x)\). Consider the sentence-level objective
\[
    \mathcal J_s(\pi)
    =
    \sum_{y\in\mathcal Y(x)}
    \pi_{\mathrm{old}}(y\mid x)
    \left[
        A(x,y)\log \pi(y\mid x)
        -
        \frac{\beta}{2}
        \left(
            \log
            \frac{
                \pi(y\mid x)
            }{
                \pi_{\mathrm{old}}(y\mid x)
            }
        \right)^2
    \right],
\]
optimized over distributions \(\pi(\cdot\mid x)\in\Delta(\mathcal Y(x))\),
with \(\beta>0\). Let \(\pi^\star(\cdot\mid x)\) be an interior stationary
point of \(\mathcal J_s\). If its normalization multiplier satisfies
\(\tau_s> 0\), then for every \(y\in\mathcal Y(x)\),
\begin{align}
    \pi^\star(y\mid x)
    =
    \frac{
        \pi_{\mathrm{old}}(y\mid x)
    }{
        \tau_s
    }
    W_0
    \left(
        \tau_s
        \exp\left(\frac{A(x,y)}{\beta}\right)
    \right),
    \label{eq:mopo-lambert-sentence}
\end{align}
where \(W_0\) is the principal branch of the Lambert \(W\) function, and
\(\tau_s\) is chosen so that
\[
    \sum_{y\in\mathcal Y(x)} \pi^\star(y\mid x)=1.
\]
In the limit \(\tau_s\to 0\), Eq.~\eqref{eq:mopo-lambert-sentence}
recovers
\[
    \pi^\star(y\mid x)
    \propto
    \pi_{\mathrm{old}}(y\mid x)
    \exp\left(\frac{A(x,y)}{\beta}\right).
\]
\end{proposition}

This proposition is the main lens through which we analyze ratio-free
off-policy objectives. It shows that, once the advantage \(A\) is fixed,
the objective does not generally target the usual exponential
KL-regularized policy. Instead, it induces a Lambert-tempered target whose
shape is controlled by the normalization multiplier \(\tau_s\). Before
studying the regimes of this multiplier, we first show that this objective
is closely connected to optimal advantage regression losses \cite{brantley2026accelerating, llmscanlearntoreasonfromoffpolicydata}.

\paragraph{Connection to $A^\star$-PO and OAPL.}
These losses can be viewed as a squared regression objective on the sentence-level
log-ratio between the learned policy and the behavior policy. To see this,
define
\[
    \ell_\theta(y_i\mid x)
    =
    \log
    \frac{
        \pi_\theta(y_i\mid x)
    }{
        \pi_{\mathrm{old}}(y_i\mid x)
    }.
\]
For a fixed sample \(y_i\) and a fixed advantage \(\widehat A_i\), the
corresponding term in our objective satisfies
\begin{align*}
    \widehat A_i
    \log \pi_\theta(y_i\mid x)
    -
    \frac{\beta}{2}
    \left(
        \log
        \frac{
            \pi_\theta(y_i\mid x)
        }{
            \pi_{\mathrm{old}}(y_i\mid x)
        }
    \right)^2
    &=
    -
    \frac{1}{2\beta}
    \left(
        \beta \ell_\theta(y_i\mid x)
        -
        \widehat A_i
    \right)^2
    +
    C_i,
\end{align*}
where \(C_i\) do not depend on \(\theta\). Therefore,
maximizing our regularized weighted log-likelihood objective is
equivalent, up to constants and a positive rescaling, to minimizing
\[
    \mathbb E_{y_{1:G}\sim\pi_{\mathrm{old}}(\cdot\mid x)}
    \left[
        \frac{1}{G}
        \sum_{i=1}^G
        \left(
            \beta
            \log
            \frac{
                \pi_\theta(y_i\mid x)
            }{
                \pi_{\mathrm{old}}(y_i\mid x)
            }
            -
            \widehat A_i
        \right)^2
    \right].
\]
This is precisely the $A^\star$-PO / OAPL form when \(\widehat A_i\) is chosen as the
log-sum-exp centered advantage
\[
    \widehat A_i^\beta
    =
    r_i
    -
    \beta
    \log
    \left(
        \frac{1}{G}
        \sum_{j=1}^G
        \exp(r_j/\beta)
    \right).
\]

Thus, OAPL is not a separate objective from the one analyzed above; it is
a particular instantiation obtained by a specific choice of advantage.
Consequently, Proposition~\ref{prop:mopo-lambert-sentence} characterizes
the target policy induced by OAPL as well.

\section{Practical Off-Policy Optimization is Pessimistic}
\label{sec:pess}

\subsection{Analysis of the objective}

Our previous proposition shows that the Lambert target policy is governed by the
density ratio $\rho_s^\star(y\mid x)
    =
        \pi^\star(y\mid x)/
        \pi_{\mathrm{old}}(y\mid x)$.
At the optimal solution, this ratio satisfies
\begin{align}
    \log \rho_s^\star(y\mid x)
    +
    \tau_s \rho_s^\star(y\mid x)
    =
    \frac{A(x,y)}{\beta}.
    \label{eq:rho-stationarity}
\end{align}
Thus \(\tau_s\) is not merely a normalization constant: it determines the
shape of the target policy. The sign of \(\tau_s\) is controlled by the
normalization of the exponential tilt. Define
\[
    Z_{\exp}(x)
    =
    \sum_{y\in\mathcal Y(x)}
    \pi_{\mathrm{old}}(y\mid x)
    \exp\left(\frac{A(x,y)}{\beta}\right).
\]
Then, on the principal branch,
\[
    \tau_s > 0
    \quad\Longleftrightarrow\quad
    Z_{\exp}(x)>1, \qquad
    \tau_s = 0
    \quad\Longleftrightarrow\quad
    Z_{\exp}(x)=1,
\]
and, whenever a real stationary solution exists,
\[
    \tau_s < 0
    \quad\Longleftrightarrow\quad
    Z_{\exp}(x)<1.
\]
This means that the way the advantage $A(x,y)$ is centered \emph{directly determines}
which regime the target policy lies in. In particular, adding a
context-dependent shift \(b(x)\) to the advantage rescales the exponential
mass by \(e^{-b(x)/\beta}\). Unlike in standard policy-gradient
estimators, where such shifts can reduce variance without changing the objective, such shifts are not innocuous here: they change the target
distribution optimized by the off-policy log-likelihood objective.

\paragraph{The pessimistic regime: \(\tau_s>0\).}
When \(\tau_s>0\), the Lambert target is more conservative than the
exponential, KL-regularized target. Indeed, for every response \(y\),
\[
    \rho_s^\star(y\mid x)
    =
    \frac{1}{\tau_s}
    W_0\left(
        \tau_s
        \exp\left(\frac{A(x,y)}{\beta}\right)
    \right)
    \le
    \exp\left(\frac{A(x,y)}{\beta}\right),
\]
since \(W_0(z)\le z\) for \(z\ge 0\). Moreover, \(W_0(z)\) is linear near
the origin:
\[
    W_0(z) = z + O(z^2)
    \quad\text{as } z\to 0,
\]
and grows only logarithmically for large \(z\)
\[
    W_0(z)
    =
    \log z - \log\log z + o(1)
    \quad\text{as } z\to\infty.
\]
Therefore low-advantage responses are treated similarly to the
exponential update, whereas very high-advantage responses are strongly
tempered. This is precisely the desired pessimistic behavior: the policy
does not over-commit exponentially to a small number of high-reward
off-policy samples, which helps protect against over-optimization and
entropy collapse \cite{huang2025correcting}. 

The special case \(\tau_s=1\) gives the particularly simple target
\[
    \rho_s^\star(y\mid x)
    =
    W_0\left(
        \exp\left(\frac{A(x,y)}{\beta}\right)
    \right).
\]

For large positive advantages, this ratio grows approximately linearly in
\(A(x,y)/\beta\), rather than exponentially. Thus \(\tau_s=1\) implements a
simple form of pessimism \cite{jin2021pessimism}: it preserves the ordering induced by the advantage while
substantially reducing the incentive to collapse onto the largest
estimated rewards. This policy was demonstrated to be minimax for offline preference learning \cite{huang2025correcting}.

\paragraph{The exponential boundary: \(\tau_s=0\).}
The case \(\tau_s=0\) is the boundary between pessimistic and optimistic
targets. Eq.~\eqref{eq:rho-stationarity} reduces to
\[
    \rho_s^\star(y\mid x)
    =
    \exp\left(\frac{A(x,y)}{\beta}\right).
\]
Thus the stationary policy is the usual exponential-tilting target
associated with KL-regularized policy optimization,
\[
    \pi^\star(y\mid x)
    \propto
    \pi_{\mathrm{old}}(y\mid x)
    \exp\left(\frac{A(x,y)}{\beta}\right).
\]
However, in the present objective, reaching exactly this regime requires
the exponential tilt to be correctly normalized at the population level:
\[
    \sum_y
    \pi_{\mathrm{old}}(y\mid x)
    \exp\left(\frac{A(x,y)}{\beta}\right)
    =
    1.
\]
This makes the \(\tau_s=0\) target delicate. Small errors in the
advantage baseline can move the population objective into either the
pessimistic regime \(\tau_s>0\) or the unstable regime \(\tau_s<0\).
Moreover, the exponential target is highly sensitive to large positive
advantage estimates, since its density ratio grows as
\(\exp(A(x,y)/\beta)\). This sensitivity is one reason why directly
targeting the KL-regularized update can be brittle in large-scale
off-policy training.

\paragraph{The unstable regime: \(\tau_s<0\).}
When \(\tau_s<0\), write \(\tau_s=-|\tau_s|\). The stationarity equation
becomes
\begin{equation}
    \log \rho_s^\star(y\mid x)+\tau_s\rho_s^\star(y\mid x)
    =
    \frac{A(x,y)}{\beta}.
    \label{eq:rho-stationarity_uns}
\end{equation}
Equivalently,
\[
    \rho_s^\star(y\mid x)
    =
    \exp\left(\frac{A(x,y)}{\beta}\right)
    \exp\left(
        |\tau_s|\rho_s^\star(y\mid x)
    \right).
\]
Thus, whenever the principal-branch solution exists,
\[
    \rho_s^\star(y\mid x)
    \ge
    \exp\left(\frac{A(x,y)}{\beta}\right).
\]
The target is therefore more aggressive than the exponential update: it
amplifies high-advantage responses even more strongly than the
KL-regularized target. This regime is also numerically fragile. From
Equation~\eqref{eq:rho-stationarity_uns},
\[
    \frac{
        \partial \rho_s^\star(y\mid x)
    }{
        \partial (A(x,y)/\beta)
    }
    =
    \frac{
        \rho_s^\star(y\mid x)
    }{
        1+\tau_s \rho_s^\star(y\mid x)
    }.
\]
For \(\tau_s<0\), the denominator can approach zero, causing the target
ratio to become extremely sensitive to small errors in the advantage. This makes
\(\tau_s<0\) a regime to avoid: it is more peaked than the exponential
target, more sensitive to advantage noise, and more prone to rapid
entropy collapse.

\subsection{The nominal advantage regression target is in the unstable regime}

As discussed above, $A^\star$-PO and OAPL instantiate the regularized off-policy objective with a specific choice of advantage. The previous analysis shows that advantage normalization strongly influences the induced target policy, and is therefore a target-design choice.  OAPL uses the log-sum-exp centered advantage
\[
    \widehat A_i^\beta
    =
    r_i
    -
    \beta
    \log
    \left(
        \frac{1}{G}
        \sum_{j=1}^G
        \exp(r_j/\beta)
    \right)\,, \quad  A^\beta(x,y) = \mathbb E_{\pi_{\mathrm{old}}}[
        \widehat A_i^\beta
        \mid y_i=y,x
    ]
\]
This gives by Jensen's inequality:
\begin{align*}
    \mathbb E_{y\sim\pi_{\mathrm{old}}(\cdot\mid x)}
    \left[
        \exp\left(
            \frac{A^\beta(x,y)}{\beta}
        \right)
    \right]
    \le
    \mathbb E
    \left[
        \exp\left(
            \frac{\widehat A_i^\beta}{\beta}
        \right)
    \right]
    =
    1 \implies \tau_s \le 0\,.
\end{align*}

In non-degenerate finite-\(G\) settings, the inequality is \emph{strict}. Hence
the target induced by this biased estimator
lies in the unstable, negative-\(\tau_s\) regime, even though the empirical group is
exponentially normalized.

This explains why the implementation choice of OAPL deviates from the
nominal theory. In the experimental setup, OAPL decouples the temperature
\(\beta_1\) used for regularization from the temperature \(\beta_2\) used
for advantage computation, choosing \(\beta_2 \gg \beta_1\). This choice
makes the advantage closer to a centered group advantage and pushes the induced target back toward the stable positive-\(\tau_s\) regime. This is discussed more in Appendix~\ref{app:oapl_regime}. This implicitly
implements the pessimism needed to stabilize off-policy learning. Rather than relying on this heuristic and an additional hyperparameter, we propose to target the Lambert policy, which was demonstrated to be more stable in RL settings \cite{huang2025correcting} and closer to a $\chi$-square divergence, implementing the principle of pessimism in the face of uncertainty \cite{jin2021pessimism, NEURIPS2024_9379ea6b}.

\subsection{Lambert Policy as a target}

The preceding analysis suggests a simple design principle: instead of
trying to target the fragile exponential boundary \(\tau_s=0\), choose a
positive Lambert multiplier by construction. Specifically, we want to be in the regime
\(\tau_s>1\) to benefit from the induced pessimism. Recall that the Lambert target ratio is
\[
    \rho^\star(y\mid x)
    =
    \frac{\pi^\star(y\mid x)}
    {\pi_{\mathrm{old}}(y\mid x)}
    =
    \frac{1}{\tau_s}
    W_0
    \left(
        \tau_s
        \exp\left(\frac{A(x,y)}{\beta}\right)
    \right),
\]
where \(\tau_s\) is chosen so that $\mathbb E_{y\sim\pi_{\mathrm{old}}(\cdot\mid x)}
    \left[
        \rho^\star(y\mid x)
    \right]
    =
    1.$

We give the simple sufficient condition on the advantage to obtain $\tau_s \ge 1$ in the following proposition.

\begin{proposition}[Shifted mean baseline yields a pessimistic Lambert target]
\label{prop:mean-minus-beta}
Fix a context \(x\) and let \(\beta>0\). For a sentence-level advantage
\(A(x,y)\), define the Lambert mass function
\[
    M_A(\tau)
    =
    \mathbb E_{y\sim\pi_{\mathrm{old}}(\cdot\mid x)}
    \left[
        \frac{1}{\tau}
        W_0
        \left(
            \tau
            \exp\left(\frac{A(x,y)}{\beta}\right)
        \right)
    \right].
\]
Let \(\tau_s\) be the positive normalization multiplier satisfying $M_A(\tau_s)=1$. We have:
\[
    \mathbb E_{y\sim\pi_{\mathrm{old}}(\cdot\mid x)}
    [A(x,y)]
    \ge
    \beta \implies\tau_s\ge 1.
\]
Consequently, the following choice of the advantage induces a Lambert multiplier satisfying $\tau_s\ge 1$. 
\[
    \widehat A_i
    =
    r_i-\bar r+\beta,
\]
\end{proposition}

\section{Experiments}
\label{sec:experiments}

We empirically evaluate whether directly targeting the pessimistic
Lambert regime improves the stability of off-policy
learning for language-model reasoning. Our experiments compare OAPL with
our Lambert objective and stress test them under small regularization strength and large policy lag.

All experiments start from Qwen3-8B~\cite{yang2025qwen3} as the initial policy. We train on
the DeepScaleR dataset~\cite{deepscaler2025}.  For training, we use four H100 GPUs, with one prompt per device, group size \(G=4\), maximum generation length \(T=4096\), and AdamW~\cite{loshchilov2018decoupled} with a learning rate of \(10^{-6}\).

We report evaluation reward as average accuracy on a held-out mixture of
math benchmarks consisting of Minerva~\cite{NEURIPS2022_18abbeef}, AIME2025, AMC23, and BRUMO~\cite{brumo}.
We also report the average generation entropy, which
measures whether the policy maintains diversity or collapses during
training.

For each $\beta > 0$,  we compare two ratio-free off-policy objectives that use the same
regularized log-ratio regression form, but differ in how the advantage is
defined. We compare the following advantages:
\begin{align*}
    & \widehat A_i^{\mathrm{OAPL}}
    =
    r_i
    -
    \beta
    \log
    \left(
        \frac{1}{G}
        \sum_{j=1}^G
        \exp(r_j/\beta)
    \right), \qquad
\widehat A_i^{\mathrm{Lambert}}
    =
    r_i-\bar r+\beta\,.
\end{align*}
The Lambert choice directly targets the pessimistic Lambert regime. As shown in
Proposition~\ref{prop:mean-minus-beta}, it guarantees that the induced
Lambert multiplier satisfies \(\tau_s\ge 1\).

To study off-policy learning, rollouts are generated by a lagged behavior
policy \(\pi_{\mathrm{old}}\), while updates are applied to the current
policy \(\pi_\theta\). The policy lag \(L\) denotes the number of
optimization steps between the behavior policy used to generate the data
and the current policy being updated. Larger \(L\) therefore corresponds
to more stale, more off-policy data. 

First, we set $L = 32$, and we study sensitivity to the regularization strength. We compare
OAPL and Lambert for $\beta \in \{10^{-1},10^{-2},10^{-3}\}$.
Figure~\ref{fig:experiments} (Top) reports evaluation reward and entropy during
training. Lambert remains stable across all three values of \(\beta\),
whereas OAPL becomes unstable for small \(\beta\). In particular, when
the regularization is weak, OAPL exhibits rapid entropy collapse and
degraded reward, while Lambert maintains higher entropy and continues to
improve. This supports the theoretical prediction that the Lambert target
is less sensitive to aggressive regularization because large advantages
are tempered rather than amplified exponentially.

Second, we study robustness to stale rollouts by varying the policy lag $L \in \{32,128,256\}$. Figure~\ref{fig:experiments} (Bottom) reports evaluation reward and entropy for
each lag. Lambert is consistently more stable as the lag increases. It
maintains reasonable entropy and converges to higher reward, while OAPL
becomes highly sensitive to stale data. Depending on the lag, OAPL either
collapses to low entropy or exhibits unstable entropy growth, both of
which are reflected in degraded or unstable reward curves.

The Lambert advantage makes off-policy learning more robust. The improvement is pronounced in the regimes predicted by our analysis to be difficult
for off-policy learning: small \(\beta\), where exponential targets
become sharp, and large policy lag, $\beta$ small is necessary to allow learning. These results support
our main claim that targeting the pessimistic Lambert regime stabilizes
off-policy learning for LLM reasoning.



\begin{figure}[t]
    \centering
    \includegraphics[width=0.77\linewidth]{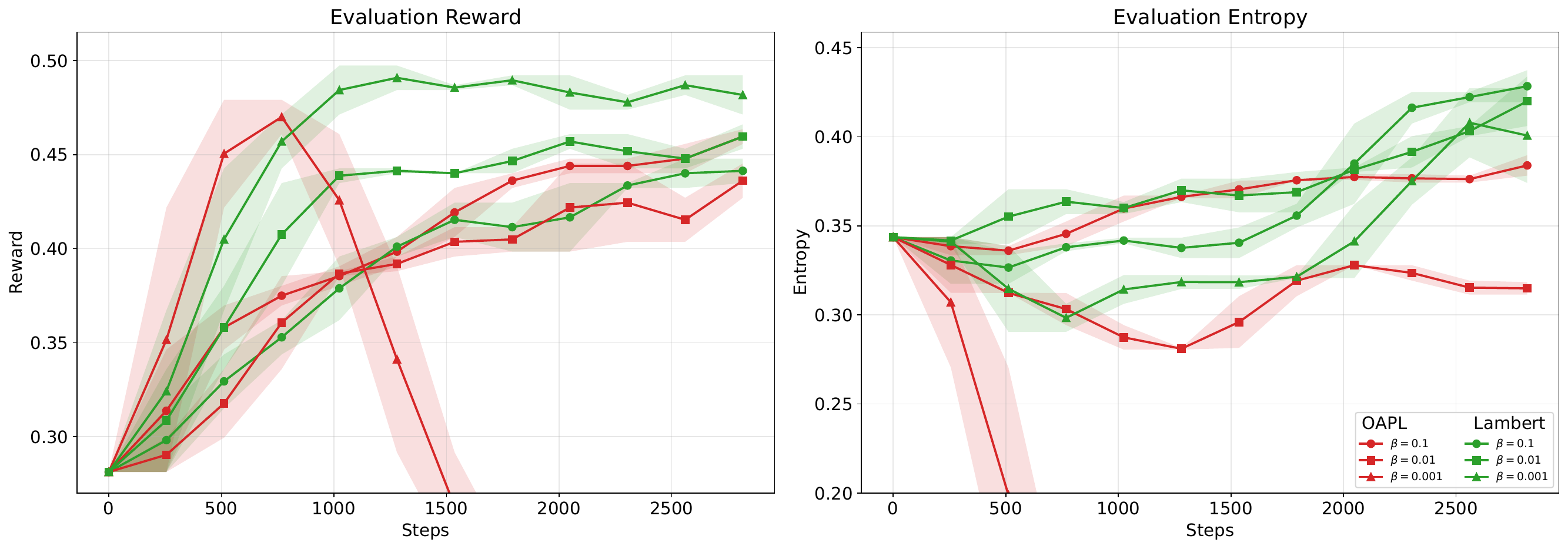}
    \vspace{-0.3em}

    \includegraphics[width=0.77\linewidth]{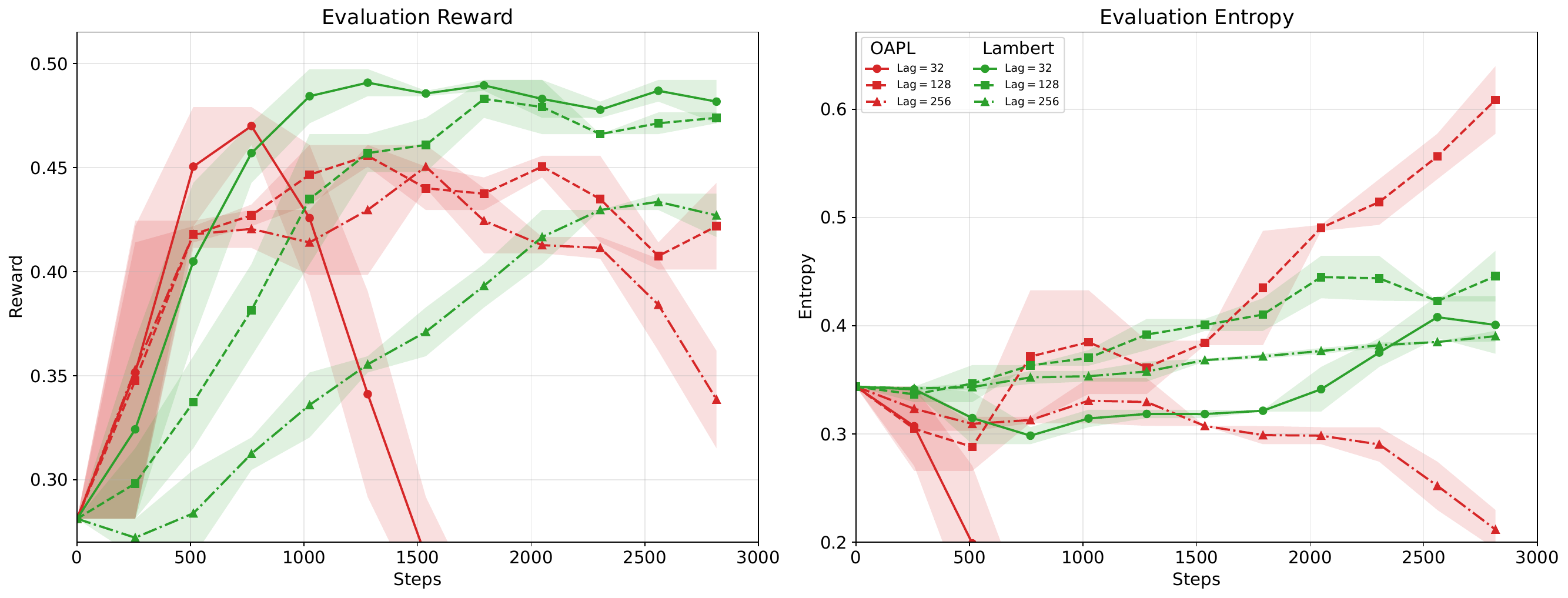}
    \vspace{-0.5em}
    \caption{\textbf{Lambert improves off-policy stability.}
    Top: Lambert remains stable while OAPL
    collapses for small \(\beta\). Bottom: Lambert maintains higher
    reward and stable entropy under stale rollouts, while OAPL becomes unstable.}
    \label{fig:experiments}
    \vspace{-1.2em}
\end{figure}



\section{Conclusion}
\label{sec:conclusion}

By characterizing regularized, ratio-free off-policy objectives as inducing Lambert-tempered policies, we showed that advantage normalization is not a benign baseline choice but a mechanism that determines whether the update is pessimistic, exponential, or unstable. This lens explains both the empirical strength and the fragility of optimal advantage regression methods: their nominal log-sum-exp advantage can induce an aggressive population target, while practical temperature decoupling implicitly restores a more conservative update. Building on this understanding, we proposed a shifted-mean advantage that targets the pessimistic Lambert regime by construction, improving robustness to regularization strength and policy lag in our experiments. We view token-level variants of the theory as a promising direction for making off-policy reasoning algorithms more principled and reliable.

\section{Acknowledgment}
\label{sec:ack}
This project was provided with computational AI and storage
resources by GENCI at IDRIS thanks to the grant 2025-A0191016862 on the supercomputer Jean Zay
(V100/A100/H100 partitions).

\bibliographystyle{unsrt}
\bibliography{reference}

@inproceedings{
huang2025correcting,
title={Correcting the Mythos of {KL}-Regularization: Direct Alignment without Overoptimization via Chi-Squared Preference Optimization},
author={Audrey Huang and Wenhao Zhan and Tengyang Xie and Jason D. Lee and Wen Sun and Akshay Krishnamurthy and Dylan J Foster},
booktitle={The Thirteenth International Conference on Learning Representations},
year={2025},
url={https://openreview.net/forum?id=hXm0Wu2U9K}
}

@inproceedings{
brantley2026accelerating,
title={Accelerating {RL} for {LLM} Reasoning with Optimal Advantage Regression},
author={Kiant{\'e} Brantley and Mingyu Chen and Zhaolin Gao and Jason D. Lee and Wen Sun and Wenhao Zhan and Xuezhou Zhang},
booktitle={The Thirty-ninth Annual Conference on Neural Information Processing Systems},
year={2026},
url={https://openreview.net/forum?id=T1V8BJO0iG}
}

@misc{llmscanlearntoreasonfromoffpolicydata,
      title={LLMs Can Learn to Reason From Off-Policy Data}, 
      author={Daniel Ritter and Owen Oertell and Bradley Guo and Jonathan Chang and Kianté Brantley and Wen Sun},
      year={2026},
      eprint={2505.20686},
      archivePrefix={arXiv},
      primaryClass={cs.LG},
      url={https://arxiv.org/abs/2505.20686}, 
}

@inproceedings{
roirl,
title={Roi{RL}: Efficient, Self-Supervised Reasoning with Offline Iterative Reinforcement Learning},
author={Aleksei Arzhantsev and Otmane Sakhi and Flavian Vasile},
booktitle={NeurIPS 2025 Workshop on Efficient Reasoning},
year={2025},
url={https://openreview.net/forum?id=PeJ1eGGygZ}
}

@inproceedings{jin2021pessimism,
  title={Is pessimism provably efficient for offline rl?},
  author={Jin, Ying and Yang, Zhuoran and Wang, Zhaoran},
  booktitle={International Conference on Machine Learning},
  pages={5084--5096},
  year={2021},
  organization={PMLR}
}

@inproceedings{NEURIPS2024_9379ea6b,
 author = {Sakhi, Otmane and Aouali, Imad and Alquier, Pierre and Chopin, Nicolas},
 booktitle = {Advances in Neural Information Processing Systems},
 doi = {10.52202/079017-2566},
 editor = {A. Globerson and L. Mackey and D. Belgrave and A. Fan and U. Paquet and J. Tomczak and C. Zhang},
 pages = {80706--80755},
 publisher = {Curran Associates, Inc.},
 title = {Logarithmic Smoothing for Pessimistic Off-Policy Evaluation, Selection and Learning},
 url = {https://proceedings.neurips.cc/paper_files/paper/2024/file/9379ea6ba7a61a402c7750833848b99f-Paper-Conference.pdf},
 volume = {37},
 year = {2024}
}

@misc{deepseek-math,
  author = {Zhihong Shao and Peiyi Wang and Qihao Zhu and Runxin Xu and Junxiao Song and Mingchuan Zhang and Y. K. Li and Y Wu and Daya Guo},
  title = {DeepSeekMath: Pushing the Limits of Mathematical Reasoning in Open Language Models},
  journal = {CoRR},
  volume = {abs/2402.03300},
  year = {2024},
  url = {https://arxiv.org/abs/2402.03300},
}

@misc{openai2026openaio1card,
      title={OpenAI o1 System Card}, 
      author={OpenAI},
      year={2026},
      eprint={2412.16720},
      archivePrefix={arXiv},
      primaryClass={cs.AI},
      url={https://arxiv.org/abs/2412.16720}, 
}

@misc{deepseekai2025,
      title={DeepSeek-R1: Incentivizing Reasoning Capability in LLMs via Reinforcement Learning}, 
      author={DeepSeek-AI},
      year={2025},
      eprint={2501.12948},
      archivePrefix={arXiv},
      primaryClass={cs.CL},
      url={https://arxiv.org/abs/2501.12948}, 
}

@inproceedings{cql,
author = {Kumar, Aviral and Zhou, Aurick and Tucker, George and Levine, Sergey},
title = {Conservative Q-learning for offline reinforcement learning},
year = {2020},
isbn = {9781713829546},
publisher = {Curran Associates Inc.},
address = {Red Hook, NY, USA},
booktitle = {Proceedings of the 34th International Conference on Neural Information Processing Systems},
articleno = {100},
numpages = {13},
location = {Vancouver, BC, Canada},
series = {NIPS '20}
}

@inproceedings{
noukhovitch2025faster,
title={Faster, More Efficient {RLHF} through Off-Policy Asynchronous Learning},
author={Michael Noukhovitch and Shengyi Huang and Sophie Xhonneux and Arian Hosseini and Rishabh Agarwal and Aaron Courville},
booktitle={The Thirteenth International Conference on Learning Representations},
year={2025},
url={https://openreview.net/forum?id=FhTAG591Ve}
}

@inproceedings{
yu2026dapo,
title={{DAPO}: An Open-Source {LLM} Reinforcement Learning System at Scale},
author={Qiying Yu and Zheng Zhang and Ruofei Zhu and Yufeng Yuan and Xiaochen Zuo and YuYue and Weinan Dai and Tiantian Fan and Gaohong Liu and Juncai Liu and LingJun Liu and Xin Liu and Haibin Lin and Zhiqi Lin and Bole Ma and Guangming Sheng and Yuxuan Tong and Chi Zhang and Mofan Zhang and Ru Zhang and Wang Zhang and Hang Zhu and Jinhua Zhu and Jiaze Chen and Jiangjie Chen and Chengyi Wang and Hongli Yu and Yuxuan Song and Xiangpeng Wei and Hao Zhou and Jingjing Liu and Wei-Ying Ma and Ya-Qin Zhang and Lin Yan and Yonghui Wu and Mingxuan Wang},
booktitle={The Thirty-ninth Annual Conference on Neural Information Processing Systems},
year={2026},
url={https://openreview.net/forum?id=2a36EMSSTp}
}

@inproceedings{
rafailov2023direct,
title={Direct Preference Optimization: Your Language Model is Secretly a Reward Model},
author={Rafael Rafailov and Archit Sharma and Eric Mitchell and Christopher D Manning and Stefano Ermon and Chelsea Finn},
booktitle={Thirty-seventh Conference on Neural Information Processing Systems},
year={2023},
url={https://openreview.net/forum?id=HPuSIXJaa9}
}

@inproceedings{
loshchilov2018decoupled,
title={Decoupled Weight Decay Regularization},
author={Ilya Loshchilov and Frank Hutter},
booktitle={International Conference on Learning Representations},
year={2019},
url={https://openreview.net/forum?id=Bkg6RiCqY7},
}

@article{brumo,
      title={Beyond Benchmarks: MathArena as an Evaluation Platform for Mathematics with LLMs}, 
      author={Jasper Dekoninck and Nikola Jovanović and Tim Gehrunger and Kári Rögnvalddson and Ivo Petrov and Chenhao Sun and Martin Vechev},
      year={2026},
      eprint={2605.00674},
      archivePrefix={arXiv},
      primaryClass={cs.CL},
      url={https://arxiv.org/abs/2605.00674}, 
}

@misc{ppo,
      title={Proximal Policy Optimization Algorithms}, 
      author={John Schulman and Filip Wolski and Prafulla Dhariwal and Alec Radford and Oleg Klimov},
      year={2017},
      eprint={1707.06347},
      archivePrefix={arXiv},
      primaryClass={cs.LG},
      url={https://arxiv.org/abs/1707.06347}, 
}

@article{yang2025qwen3,
  title={{Qwen3} Technical Report},
  author={Yang, An and Li, Anfeng and Yang, Baosong and others},
  journal={arXiv preprint arXiv:2505.09388},
  year={2025},
  url={https://arxiv.org/abs/2505.09388}
}

@misc{deepscaler2025,
  title={DeepScaleR: Surpassing O1-Preview with a 1.5B Model by Scaling RL},
  author={Michael Luo and Sijun Tan and Justin Wong and Xiaoxiang Shi and William Y. Tang and Manan Roongta and Colin Cai and Jeffrey Luo and Li Erran Li and Raluca Ada Popa and Ion Stoica},
  howpublished={https://pretty-radio-b75.notion.site/DeepScaleR-Surpassing-O1-Preview-with-a-1-5B-Model-by-Scaling-RL-19681902c1468005bed8ca303013a4e2},
  note={Notion Blog},
  year={2025}
}

@inproceedings{NEURIPS2022_18abbeef,
 author = {Lewkowycz, Aitor and Andreassen, Anders and Dohan, David and Dyer, Ethan and Michalewski, Henryk and Ramasesh, Vinay and Slone, Ambrose and Anil, Cem and Schlag, Imanol and Gutman-Solo, Theo and Wu, Yuhuai and Neyshabur, Behnam and Gur-Ari, Guy and Misra, Vedant},
 booktitle = {Advances in Neural Information Processing Systems},
 editor = {S. Koyejo and S. Mohamed and A. Agarwal and D. Belgrave and K. Cho and A. Oh},
 pages = {3843--3857},
 publisher = {Curran Associates, Inc.},
 title = {Solving Quantitative Reasoning Problems with Language Models},
 url = {https://proceedings.neurips.cc/paper_files/paper/2022/file/18abbeef8cfe9203fdf9053c9c4fe191-Paper-Conference.pdf},
 volume = {35},
 year = {2022}
}

@misc{
awc,
title={Advantage-Weighted Regression: Simple and Scalable Off-Policy Reinforcement Learning},
author={Xue Bin Peng and Aviral Kumar and Grace Zhang and Sergey Levine},
year={2021},
url={https://openreview.net/forum?id=ToWi1RjuEr8}
}

@misc{kimi,
      title={Kimi K2: Open Agentic Intelligence}, 
      author={Kimi Team},
      year={2026},
      eprint={2507.20534},
      archivePrefix={arXiv},
      primaryClass={cs.LG},
      url={https://arxiv.org/abs/2507.20534}, 
}

@InProceedings{pmlr-v37-schulman15,
  title = 	 {Trust Region Policy Optimization},
  author = 	 {Schulman, John and Levine, Sergey and Abbeel, Pieter and Jordan, Michael and Moritz, Philipp},
  booktitle = 	 {Proceedings of the 32nd International Conference on Machine Learning},
  pages = 	 {1889--1897},
  year = 	 {2015},
  editor = 	 {Bach, Francis and Blei, David},
  volume = 	 {37},
  series = 	 {Proceedings of Machine Learning Research},
  address = 	 {Lille, France},
  month = 	 {07--09 Jul},
  publisher =    {PMLR},
  url = 	 {https://proceedings.mlr.press/v37/schulman15.html}
}

\newpage
\appendix

\section{Societal Impact}
\label{sec:impact}

This work studies learning objectives for reinforcement learning in large language models. Its primary contribution is methodological: by characterizing and controlling the implicit pessimism induced by off-policy objectives, it aims to make RL training more stable and less prone to entropy collapse. More stable training may reduce computational waste and make post-training pipelines easier to reproduce and debug. At the same time, improvements in RL for reasoning can strengthen the capabilities of language models, which may have both beneficial and harmful downstream uses. Better reasoning systems can support education, scientific assistance, programming, and mathematical problem solving, but they may also improve the effectiveness of systems used for misinformation, cyber misuse, or other harmful applications. We therefore view the societal impact as indirect, mediated by how improved RL methods are incorporated into future language-model training pipelines. Responsible deployment should pair such optimization methods with standard safeguards, evaluation, and monitoring appropriate to the intended application.

\section{Proofs}

\subsection{Proof of Proposition~\ref{prop:mopo-lambert-sentence}}

\begin{proof}
Fix \(x\). For readability, we omit the conditioning on \(x\) when there
is no ambiguity, and write
\[
    \rho(y)
    =
    \frac{\pi(y\mid x)}{\pi_{\mathrm{old}}(y\mid x)}.
\]
We assume \(\pi_{\mathrm{old}}(y\mid x)>0\) on \(\mathcal Y(x)\), or
equivalently restrict \(\mathcal Y(x)\) to the support of
\(\pi_{\mathrm{old}}(\cdot\mid x)\). Since
\[
    \pi(y\mid x)
    =
    \pi_{\mathrm{old}}(y\mid x)\rho(y),
\]
the normalization constraint \(\sum_y \pi(y\mid x)=1\) becomes
\[
    \sum_{y\in\mathcal Y(x)}
    \pi_{\mathrm{old}}(y\mid x)\rho(y)
    =
    1.
\]
Moreover,
\[
    \log \pi(y\mid x)
    =
    \log \pi_{\mathrm{old}}(y\mid x)
    +
    \log \rho(y).
\]
Therefore, up to a constant independent of \(\rho\), the objective can be
written as
\[
    \mathcal J_s(\rho)
    =
    \sum_{y\in\mathcal Y(x)}
    \pi_{\mathrm{old}}(y\mid x)
    \left[
        A(x,y)\log \rho(y)
        -
        \frac{\beta}{2}
        \left(\log \rho(y)\right)^2
    \right].
\]

Let \(\pi^\star\) be an interior stationary point and let
\[
    \rho^\star(y)
    =
    \frac{\pi^\star(y\mid x)}
    {\pi_{\mathrm{old}}(y\mid x)}.
\]
Since the point is interior, \(\rho^\star(y)>0\) for all
\(y\in\mathcal Y(x)\). Introduce a Lagrange multiplier
\(\beta\tau_s\) for the normalization constraint and consider
\[
    \mathcal L(\rho,\tau_s)
    =
    \sum_y
    \pi_{\mathrm{old}}(y\mid x)
    \left[
        A(x,y)\log \rho(y)
        -
        \frac{\beta}{2}
        \left(\log \rho(y)\right)^2
    \right]
    -
    \beta\tau_s
    \left(
        \sum_y
        \pi_{\mathrm{old}}(y\mid x)\rho(y)
        -
        1
    \right).
\]
The first-order condition with respect to \(\rho(y)\) gives
\[
    \pi_{\mathrm{old}}(y\mid x)
    \left[
        \frac{A(x,y)}{\rho(y)}
        -
        \frac{\beta\log\rho(y)}{\rho(y)}
        -
        \beta\tau_s
    \right]
    =
    0.
\]
Since \(\pi_{\mathrm{old}}(y\mid x)>0\), this is equivalent to
\[
    A(x,y)
    -
    \beta\log\rho^\star(y)
    -
    \beta\tau_s\rho^\star(y)
    =
    0.
\]
Dividing by \(\beta\), we obtain the stationarity equation
\[
    \log\rho^\star(y)
    +
    \tau_s\rho^\star(y)
    =
    \frac{A(x,y)}{\beta}.
\]
Exponentiating both sides yields
\[
    \rho^\star(y)
    \exp\left(\tau_s\rho^\star(y)\right)
    =
    \exp\left(\frac{A(x,y)}{\beta}\right).
\]
If \(\tau_s>0\), multiplying by \(\tau_s\) gives
\[
    \tau_s\rho^\star(y)
    \exp\left(\tau_s\rho^\star(y)\right)
    =
    \tau_s
    \exp\left(\frac{A(x,y)}{\beta}\right).
\]
Applying the principal branch \(W_0\) of the Lambert \(W\) function,
defined by \(W_0(z)e^{W_0(z)}=z\), gives
\[
    \tau_s\rho^\star(y)
    =
    W_0
    \left(
        \tau_s
        \exp\left(\frac{A(x,y)}{\beta}\right)
    \right).
\]
Thus
\[
    \rho^\star(y)
    =
    \frac{1}{\tau_s}
    W_0
    \left(
        \tau_s
        \exp\left(\frac{A(x,y)}{\beta}\right)
    \right).
\]
Multiplying by \(\pi_{\mathrm{old}}(y\mid x)\), we obtain
\[
    \pi^\star(y\mid x)
    =
    \frac{
        \pi_{\mathrm{old}}(y\mid x)
    }{
        \tau_s
    }
    W_0
    \left(
        \tau_s
        \exp\left(\frac{A(x,y)}{\beta}\right)
    \right).
\]
The multiplier \(\tau_s\) is chosen so that the distribution is
normalized:
\[
    \sum_{y\in\mathcal Y(x)}
    \pi^\star(y\mid x)
    =
    1.
\]

Finally, consider the limit \(\tau_s\to 0\). Since
\(W_0(z)=z+o(z)\) as \(z\to 0\), we have
\[
    \frac{1}{\tau_s}
    W_0
    \left(
        \tau_s
        \exp\left(\frac{A(x,y)}{\beta}\right)
    \right)
    \longrightarrow
    \exp\left(\frac{A(x,y)}{\beta}\right).
\]
Therefore the limiting density ratio is
\[
    \rho^\star(y)
    =
    \exp\left(\frac{A(x,y)}{\beta}\right),
\]
and hence the limiting policy has the exponential-tilting form
\[
    \pi^\star(y\mid x)
    \propto
    \pi_{\mathrm{old}}(y\mid x)
    \exp\left(\frac{A(x,y)}{\beta}\right).
\]
This proves the claim.
\end{proof}

\subsection{Proof of Proposition~\ref{prop:mean-minus-beta}}

\begin{proof}

Define
\[
    M(\tau)
    =
    \mathbb E_{y\sim\pi_{\mathrm{old}}(\cdot\mid x)}
    \left[
        \frac{1}{\tau}
        W_0
        \left(
            \tau
            \exp\left(\frac{A(x,y)}{\beta}\right)
        \right)
    \right].
\]
For \(\tau>0\), \(M(\tau)\) is decreasing in \(\tau\). Therefore, to
guarantee \(\tau_s\ge 1\), it is enough to ensure
\[
    M(1)\ge 1,
\]
that is,
\[
    \mathbb E_{y\sim\pi_{\mathrm{old}}(\cdot\mid x)}
    \left[
        W_0
        \left(
            \exp\left(\frac{A(x,y)}{\beta}\right)
        \right)
    \right]
    \ge 1.
\]

Let \(r(y)\) be the raw reward and choose the shifted empirical mean baseline with $G$ completions from $\pi_{\mathrm{old}}$:
\[
    b(x)
    =
    \bar r_G
    -
    \beta.
\]
Then
\[
    A(x,y)
    =
    r(y)-b(x)
    =
    r(y)
    -
    \mathbb E_{\pi_{\mathrm{old}}}[r(y)]
    +
    \beta,
\]
so
\[
    \mathbb E_{\pi_{\mathrm{old}}}[A(x,y)]
    =
    \beta.
\]
Now define
\[
    f(u)=W_0(e^u).
\]
This function is convex, since if \(w=f(u)\), then
\[
    f''(u)
    =
    \frac{w}{(1+w)^3}
    >
    0.
\]
Therefore, by Jensen's inequality,
\begin{align*}
    \mathbb E_{\pi_{\mathrm{old}}}
    \left[
        W_0
        \left(
            \exp\left(\frac{A(x,y)}{\beta}\right)
        \right)
    \right]
    &=
    \mathbb E_{\pi_{\mathrm{old}}}
    \left[
        f\left(\frac{A(x,y)}{\beta}\right)
    \right]
    \\
    &\ge
    f\left(
        \mathbb E_{\pi_{\mathrm{old}}}
        \left[
            \frac{A(x,y)}{\beta}
        \right]
    \right)
    \\
    &=
    f(1)
    =
    W_0(e)
    =
    1.
\end{align*}
Hence \(M(1)\ge 1\), and therefore
\[
    \tau_s\ge 1.
\]

In practice, with a sampled group \(r_1,\dots,r_G\), we use
\[
    \bar r
    =
    \frac{1}{G}\sum_{i=1}^G r_i,
    \qquad
    \widehat b(x)
    =
    \bar r-\beta.
\]
This gives the empirical advantage
\[
    \widehat A_i
    =
    r_i-\widehat b(x)
    =
    r_i-\bar r+\beta.
\]
Since
\[
    \frac{1}{G}\sum_{i=1}^G \widehat A_i
    =
    \beta,
\]
the same Jensen argument gives
\begin{align*}
    \frac{1}{G}
    \sum_{i=1}^G
    W_0
    \left(
        \exp\left(\frac{\widehat A_i}{\beta}\right)
    \right)
    &\ge
    W_0
    \left(
        \exp
        \left(
            \frac{1}{G}
            \sum_{i=1}^G
            \frac{\widehat A_i}{\beta}
        \right)
    \right)
    \\
    &=
    W_0(e)
    =
    1.
\end{align*}
Thus the empirical target satisfies
\[
    \widehat \tau_s\ge 1.
\]

For the population objective with the group baseline, the
effective advantage is
\[
    A_G(x,y)
    =
    \mathbb E
    \left[
        r_i-\bar r+\beta
        \mid y_i=y,x
    \right]
    =
    \frac{G-1}{G}
    \left(
        r(y)-V_{\mathrm{old}}(x)
    \right)
    +
    \beta,
\]
where
\[
    V_{\mathrm{old}}(x)
    =
    \mathbb E_{y\sim\pi_{\mathrm{old}}(\cdot\mid x)}[r(y)].
\]
This also satisfies
\[
    \mathbb E_{\pi_{\mathrm{old}}}[A_G(x,y)]
    =
    \beta,
\]
so the same Jensen argument guarantees that the induced population target
satisfies
\[
    \tau_s\ge 1.
\]

Therefore, the simple practical rule
\[
    \widehat A_i = r_i-\bar r+\beta
\]
places the Lambert target in the desired pessimistic regime
\(\tau_s\ge 1\). 

\end{proof}
\section{Discussions}

\subsection{Large-\(\beta_2\) limit of the OAPL advantage}\label{app:oapl_regime}

In the experimental setup, OAPL decouples the temperature
\(\beta_1\) used for regularization from the temperature \(\beta_2\) used
for advantage computation, choosing \(\beta_2 \gg \beta_1\). This choice
makes the advantage closer to a centered group advantage and puts it in the nice regime.  Indeed, letting
\[
    \bar r = \frac{1}{G}\sum_{j=1}^G r_j,
    \qquad
    \widehat{\mathrm{Var}}_G(r)
    =
    \frac{1}{G}\sum_{j=1}^G (r_j-\bar r)^2,
\]
we have, for large \(\beta_2\),
\begin{align*}
    \widehat A_i^{\beta_2}
    &=
    r_i
    -
    \beta_2
    \log
    \left(
        \frac{1}{G}
        \sum_{j=1}^G
        \exp(r_j/\beta_2)
    \right)
    \\
    &=
    r_i
    -
    \left(
        \bar r
        +
        \frac{1}{2\beta_2}
        \widehat{\mathrm{Var}}_G(r)
        +
        O(\beta_2^{-2})
    \right)
    \\
    &=
    r_i-\bar r
    -
    \frac{1}{2\beta_2}
    \widehat{\mathrm{Var}}_G(r)
    +
    O(\beta_2^{-2})
    \\
    &\xrightarrow[\beta_2\to\infty]{}
    r_i-\bar r .
\end{align*}
Thus, as \(\beta_2\) becomes large, the log-sum-exp centered advantage
approaches the usual group-centered advantage.

At the population level, the limiting induced advantage is therefore
\begin{align*}
    A_G^\infty(x,y)
    &=
    \mathbb E
    \left[
        r_i-\bar r
        \mid y_i=y,x
    \right]
    \\
    &=
    r(y)
    -
    \frac{1}{G}
    \left(
        r(y)
        +
        (G-1)
        V_{\mathrm{old}}(x)
    \right)
    \\
    &=
    \frac{G-1}{G}
    \left(
        r(y)-V_{\mathrm{old}}(x)
    \right),
\end{align*}
where
\[
    V_{\mathrm{old}}(x)
    =
    \mathbb E_{y\sim\pi_{\mathrm{old}}(\cdot\mid x)}
    [r(y)].
\]
This induced advantage is centered under the behavior policy:
\begin{align*}
    \mathbb E_{y\sim\pi_{\mathrm{old}}(\cdot\mid x)}
    \left[
        A_G^\infty(y\mid x)
    \right]
    &=
    \frac{G-1}{G}
    \mathbb E_{y\sim\pi_{\mathrm{old}}(\cdot\mid x)}
    \left[
        r(y)-V_{\mathrm{old}}(x)
    \right]
    \\
    &=
    0.
\end{align*}
Consequently, at the regularization temperature \(\beta_1\),
\begin{align*}
    Z_G^\infty(x)
    &:=
    \mathbb E_{y\sim\pi_{\mathrm{old}}(\cdot\mid x)}
    \left[
        \exp
        \left(
            \frac{A_G^\infty(x,y)}{\beta_1}
        \right)
    \right]
    \\
    &\ge
    \exp
    \left(
        \frac{
            \mathbb E_{y\sim\pi_{\mathrm{old}}(\cdot\mid x)}
            [A_G^\infty(x,y)]
        }{\beta_1}
    \right)
    \\
    &=
    1,
\end{align*}
with strict inequality whenever the rewards are non-degenerate. Hence the
large-\(\beta_2\) limit lies in the positive-\(\tau_s\) regime. By
continuity, choosing \(\beta_2\gg\beta_1\) pushes the finite-\(\beta_2\)
target back toward this positive-\(\tau_s\) regime. This implicitly
implements the pessimism needed to stabilize off-policy learning.


\end{document}